\documentclass[10pt]{article}

\usepackage{longtable}% for long tables

\usepackage[load-configurations=version-1]{siunitx} % newer version
\usepackage{siunitx}
\usepackage{natbib}
\usepackage{hyperref}
\usepackage{nameref}
\hypersetup{colorlinks,
            linkcolor=blue,
            citecolor=blue,
            urlcolor=magenta,
            linktocpage,
            plainpages=false}
\usepackage[margin=1.5in]{geometry}
\newtheorem{theorem}{Theorem}
\newtheorem{lemma}{Lemma}
\newtheorem{proposition}{Proposition}
\newtheorem{proof}{Proof}
\newtheorem{corollary}{Corollary}
\usepackage{amsmath}
\usepackage{authblk}

\usepackage{algorithm}
\usepackage{algorithmic}
\usepackage{bbold}

\usepackage{graphicx}
\usepackage{float}
\usepackage{wrapfig}
\usepackage{subfigure}

\newcommand{\R}{\mathbb{R}}

\newcommand{\T}{\hspace{-0.25ex}\top\hspace{-0.25ex}}

\begin{document}
\title{Direction Matters: On Influence-Preserving \\ Graph Summarization and Max-cut Principle \\for Directed Graphs}

\author[1]{Wenkai Xu\footnote{Contact at wenkaix@gatsby.ucl.ac.uk}}
\author[2]{Gang Niu}
\author[3,4]{Aaop Hyv\"{a}rinen}
\author[2,5]{Masashi Sugiyama}
\affil[1]{Gatsby Unit  of Computational Neuroscience}
\affil[2]{RIKEN AIP}
\affil[3]{INRIA-Saclay}
\affil[4]{University of Helsinki}
\affil[5]{The University of Tokyo}
\date{}
\maketitle

\iffalse
\vspace{-1.8cm}
\begin{center}
\textbf{Anonymous}
\end{center}
\fi

\begin{abstract}
% Compression
Summarizing large-scaled directed graphs into small-scale representations is a useful but less studied problem setting. Conventional clustering approaches, which based on ``Min-Cut"-style criteria, compress both the vertices and edges of the graph into the communities, that lead to a loss of directed edge information. On the other hand, compressing the vertices while preserving the directed edge information provides a way to learn the small-scale representation of a directed graph.
% Reconstruction
The \textit{reconstruction error}, which measures the edge information preserved by the summarized graph, can be used to learn such representation.
Compared to the original graphs, the summarized graphs are easier to analyze and are capable of extracting group-level features which is useful for efficient interventions of population behavior. 
In this paper, we present a model, based on minimizing \textit{reconstruction error} with non-negative constraints, which relates to a ``Max-Cut" criterion that simultaneously identifies the compressed nodes and the directed compressed relations between these nodes.
A multiplicative update algorithm with column-wise normalization is proposed.
We further provide theoretical results on the identifiability of the model and on the convergence of the proposed algorithms. 
Experiments are conducted to demonstrate the accuracy and robustness of the proposed method. 
\end{abstract}

\iffalse
\fi

\section{Introduction}
%Compressing large scaled network; conventional clustering; compressing node only
%Large-scaled graphs may face extensive issues understanding the relationship between vertices. 
In directed graphs, it is important to understand the influence between vertices, which is represented by the directed edges. Investigating the influence structure in graphs has become an evolving research field that attracts wide attention from scientific communities including social sciences \citep{tang2009social, li2018social, mehmood2013csi}, economics \citep{spirtes2005graphical, jackson2011overview}, ecological sciences \citep{pavlopoulos2011using, delmas2019analysing} and more. In large-scaled densely-connected directed graphs, finding an efficient way to compress vertices and summarizing the directed influence between vertices are not only useful to visualize complicated networks but also crucial to extract group-level features for further analysis such as profiling or intervention. 

Conventional graph clustering methods group the densely connected vertices into the same community on undirected graphs \citep{fortunato2010community,schaeffer2007graph,shi2000normalized}. Directed graph clustering is commonly based on symmetrized undirected graphs \citep{malliaros2013clustering}. However, the recovered communities do not preserve much of the edge information since the communities themselves are sparsely connected.
Hence, effective reconstruction of the original graph from the summarized graph is a meaningful task that enjoys applications in graph compression \citep{dhabu2013partition,dhulipala2016compressing}, graph sampling \citep{orbanz2017subsampling, leskovec2006sampling} and so on. 

%Undirected graph summarization;  influence maximization and  other related work
For example, in a large-scaled social network, individual level connections are hard to analyze and contain a fair amount of noise. It is complicated to directly extract group-level features and interpret the influence structure of the graphs. In social network analysis, for instance, the Key Opinion Leaders (KOL) \citep{valente2007identifying,nisbet2009two} with common features may also share similar influence structure. Such information is important in terms of understanding the opinion diffusions within the network, as well as implementing interventions for various purposes such as marketing \citep{chaney2001opinion} or pooling \citep{zhou2009finding} \citep{thomson1998local}
. Moreover, extracting these features from the KOL within a group may also enable us to analyze the fairness of a certain process and perform de-bias actions when necessary.

\begin{wrapfigure}{r}{0pt}
\subfigure{\label{fig:illustration1}\includegraphics[scale=0.3]{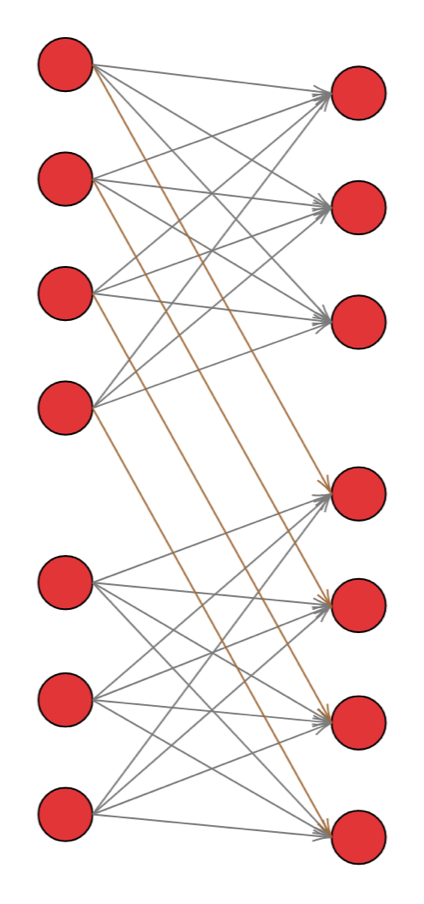}}\hspace{0.13cm}\subfigure[\label{fig:illustration2}]{\includegraphics[scale=0.3]{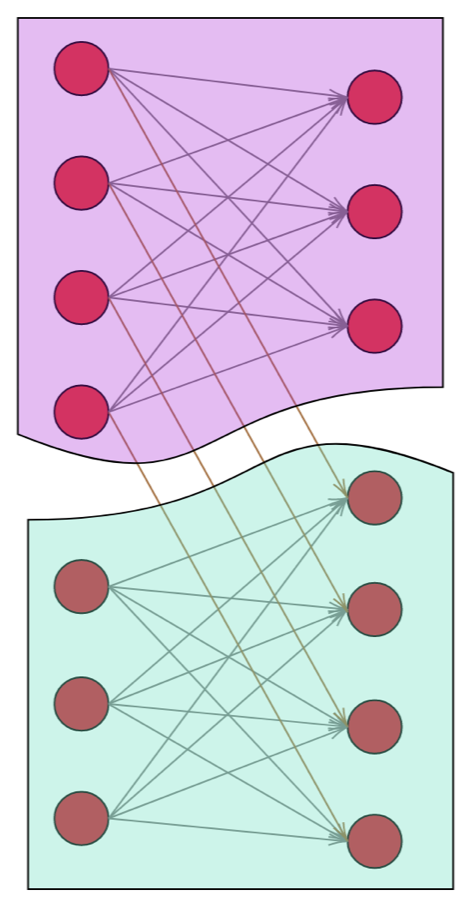}}\hspace{0.15cm}\subfigure[\label{fig:illustration3}]{\includegraphics[scale=0.3]{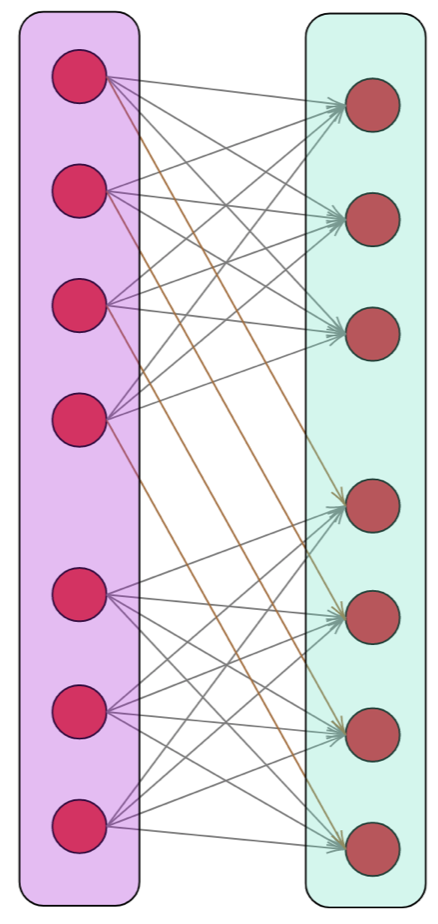}}
\caption{\label{fig:illustration}A Toy Example: a) a directed graph; b) a Min-Cut clustering; c) the desired summarization}
\end{wrapfigure}

Previous works have considered related problems in undirected graph settings \citep{shahaf2013information, navlakha2008graph}, which aim to define compressed nodes by preserving particular structures. Graph compression literature \citep{ maneth2015survey, fan2012query,dhulipala2016compressing} is also related, while the goal is to minimize the storage space, irrespective of preserving feature patterns of the graph. In addition, another line of related work, under the theme of influence maximization \citep{li2018influence}, studies directed influence of a set of vertices to the rest of the network. 
%instead of directed influence to another set of vertices

In our setting, we would like to extract sets of vertices, each becoming a compressed node, such that the influence between vertices are maximally preserved by the directed summarized graph.
Previous works such as flow-based graph summarization \citep{shi2016topic} or graph de-densification \citep{maccioni2016scalable} addressed a similar problem based on directed influence. Though these works deal with directed graphs, the directions of summarized nodes are defined from different domains so that the algorithms essentially apply to symmetrized undirected graphs. In this work, we present a novel criterion that is applicable to directed graphs, exploiting the asymmetric information of the directed edges and preserving the influence as much as possible. 

\iffalse
%Left: directed graph; Middle: min-cut cluster; Right: desired summarization; Graph Summarization via mapping $\Phi$
\begin{figure}
\subfigure[\label{fig:illustration1}]{\includegraphics[scale=0.3]{res/directed_graph_corp.png}}\hspace{0.13cm}\subfigure[\label{fig:illustration2}]{\includegraphics[scale=0.3]{res/min_cut_cluster_corp.png}}\hspace{0.15cm}\subfigure[\label{fig:illustration3}]{\includegraphics[scale=0.3]{res/summarization_corp.png}}
\hspace{0.16cm}\subfigure{\includegraphics[scale=0.14]{res/map_illustrate.png}\label{fig:mapping}}
\caption{Illustation on Directed Graph Summarization: a) a directed graph; b) Min-Cut clustering; c) desired summarization; d) graph summarization mapping}
\end{figure}

\begin{figure}
%\hspace{0.1in}
\includegraphics[scale=0.1]{res/directed_graph.png}\includegraphics[scale=0.1]{res/min_cut_cluster.png}\includegraphics[scale=0.1]{res/summarization.png}

\caption{An Example on interaction-based clustering.
Left: directed graph; Middle: clustering by Min-Cut; Right: summarization\label{fig:illustration}}
\end{figure}
\begin{figure}\centering
%\hspace{0.1in}
\includegraphics[scale=0.13]{res/map_illustrate.png}
\caption{A illustrate on graph summarization under $\Phi$ mapping\label{fig:mapping}}
\end{figure}
\fi

%Directed network  much  harder; Reconstruction error for  graph summarization; Related to Max-cut.
In directed networks, the summarization is harder as there are both the edge weights to be summarized as well as the edge directions. To effectively summarize directed graphs, we focus on the \textit{reconstruction error} from the summarized graph to the original graph. The directed graph summarization is more useful compared to the undirected case. For instance, with well-defined directed causal edges, the summarization can be helpful to approximate causal information between compressed nodes.   
Conventional clustering or dimensionality reduction methods utilizing the ``Max-Flow, Min-Cut"-style criteria, 
%compressed both vertices and edges at the same time. 
compressed vertices without considering to preserve the edge information.
These methods are unable to perform such summarization, illustrated in Figure.~\ref{fig:illustration}, since the objective is to minimize the connections between compressed nodes, which results in large reconstruction error thus undesired grouping.
Our proposed objective is closely related to but essentially different from such a scheme while we try to maximize the ``Cut" to preserve the directed edge information. Various discrete optimization schemes such as Dulmage-Mendelsohn Decomposition \citep{dulmage1958coverings} can also find a good summarization in a noiseless case, while they are less accurate and harder to implement when the noise level is high. On the other hand, our proposed model does not only work well in the noiseless case but is also more robust in the presence of noise.

This paper is organized as follows. In Section \ref{sec:prelim}, we introduce notations and the  problem setting. In Section \ref{sec:ips}, we present our learning objective and propose the Structured Non-negative Matrix Factorization (StNMF) algorithm to solve the problem. In Section \ref{sec:theory}, we provide theoretical results for reconstruction error, identifiability, and convergence of the algorithm. In Section \ref{sec:exp}, we experimentally demonstrate the usefulness of the proposed method and conclude in Section \ref{sec:conclusion}.

\section{Preliminaries and Problem Formulation} \label{sec:prelim}
In this work, we focus on simple directed graphs, which exclude self-loops and multiple edges. In this paper, we use ``graph" for referring to a directed graph when there is no ambiguity. 
%The summarized graph is assumed to be simple as well.
A positive value in the adjacency matrix represents an out-edge. We may use a negative value to represent an in-edge. The inhibition type of directed relations, where an out-edge has a negative influence, are out of the scope of this paper. We continue by defining some preliminary concepts.
\begin{table}[t!]
    \centering
    \caption{Term Comparison between Original Graph and Summarized Graph}
    \label{tab:compare-table}
    \begin{tabular}{|c|c|c|}
    \hline
    original graph: $G$    & vertices: $x_i\in V$ & edges: $e_{ij}\in E$ \\
    \hline
    summarized graph: $H$ & compressed nodes: $c_{I}\in C$ & compressed relations: $r_{IJ}\in R$ \\
    \hline
    \end{tabular}
\end{table}

\subsection{Notations and Definitions}
%define compression of vertices and edges
Denote a directed graph of the node set $V$ and the directed edge set $E$ by $G=(V,E)$. Denote a summarized directed graph of the  compressed node set $C$ and the  directed relation set $R$. In this work, both $G$ and $H$ are simple. We distinguish terms in both graphs shown in Table.\ref{tab:compare-table}.
A \emph{node-compression} is a function $\phi_V:V\to C$ that assigns a vertex $x_i \in V$ to a compressed node $c_I  \in C$. In this work, $\phi_V$ is surjective.\footnote{We do not require all vertices belongs to a compressed node as opposed to the graph partition problem.} An \emph{edge-compression} is a function $\phi_E:E\to R$.
We say an edge-compression, $\phi_E$, is induced from a node-compression $\phi_V$ if $\phi_V(x_i)=\phi_V(x_i')$, $\phi_V(x_j)=\phi_V(x_j')$, implies $ \phi_E(e_{ij})=\phi_E(e_{i'j'})$, $\forall i,j,i',j'$, i.e., vertices assigned to the same compressed node admit the same compressed relation. Hence, we can write $\phi_E(e_{ij})=r_{IJ}$, $\forall\phi_V(x_i)=c_I, \phi_V(x_j)=c_J$.
In this work, we only consider the edge-compression induced from the node-compression. 
%We use lower-case letter as index of vertex from the original graph and corresponding upper-case letter as the index for compressed node where the vertex belongs to.

%summarization
%\vspace{-.1cm}
\begin{wrapfigure}{r}{0pt}
\includegraphics[width=6cm,height=4.5cm]{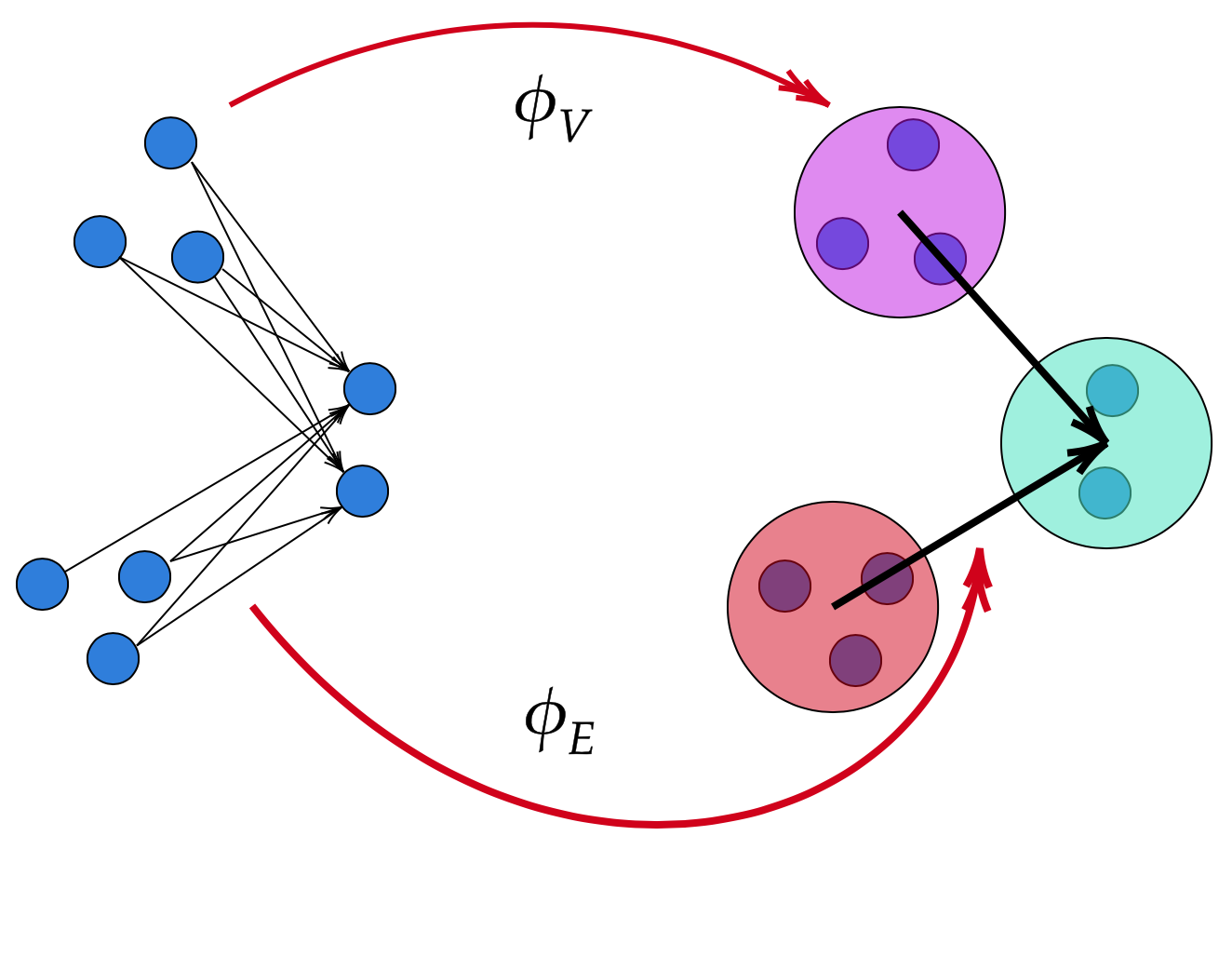}
\caption{
%\vspace{-.8cm}
Summarization $\Phi$ based on the compressions $(\phi_V, \phi_E)$\label{fig:mapping}}
\end{wrapfigure}

A graph \textit{summarization}, based on the compressions $\phi_V$ and $\phi_E$, refers to the map $\Phi$ from the original directed graph $G$ to the summarized graph $H$, such that $\Phi(G, \phi_V,\phi_E)=H$ (Illustrate in Figure.\ref{fig:mapping}). 
A graph \textit{summarization} of size $k$, $\Phi_k$ is a constrained mapping where $|C|=k \leq |V|$. In practice, we would like $k\ll|V|$. When $k=|V|$, the summarization is trivial as the original graph always gives reconstruction error zero.

%\textit{Pairwise influence} are commonly represented by the directed edges in the original graph; similarly, \textit{summarized influence} can be represented between directed edges in summarized graph.  Denote  $\ell(\cdot, \cdot )$ to be a distance measure between directed edges, and we say the distance measure is characteristic if $\ell(e, e')=0$ implies $e=e'$.
Denote $A \in \R^{n \times n}$ as the \textit{asymmetric} adjacency matrix of a directed graph such that $A_{ij}=1$ if there is a directed edge from $x_{i}$ to $x_{j}$, and $A_{ij}=0$ otherwise.
Denote $T\in\R^{n\times n}$ as the \textit{skew-symmetric} adjacency matrix of a directed graph where $T_{ij}=1$ and $T_{ji}=-1$ if there is
a directed edge from $x_{i}$ to $x_{j}$; $T_{ij}=0$ if there is
no edge connection between $x_{i}$ and $x_{j}$.
%For a simple directed graph, $A_{ij}=1$ implies $A_{ji}=0$ as there is no more than one directed edge between vertices $x_{i}$ and $x_{j}$; and $A_{ii}=0,\forall i$ as there is no self-loops. Moreover, we have $T=A-A^{\T}$ by definition. When $A$ is non-symmetric, $T$ is a nontrivial matrix, i.e., not all entries are $0$.  
We say a directed graph to be connected if its undirected skeleton is connected.

\subsection{Influence Preserving Criteria}
Consider the performance measure of our graph summarization problem. The quality of a summarization can be measured by how much the directed edge information can be recovered from the summarized graph, via the reconstruction error: 
\begin{equation}\label{general ips loss}
L_{0}(G, \phi_V,\phi_E) = \sum_{I,J}{ \sum_{x_i\in c_I, x_j\in c_J} {\ell(e_{ij}, r_{IJ})}},
\end{equation}
where $\ell$ is some non-negative loss measure. 
We use the term influence to describe the information in directed edges. By choosing different loss $\ell$, the reconstruction error describes different types of influence-preserving criteria. We say a graph has an exact \emph{Influence Preserving Structure} (IPS) if the relevant reconstruction error $L_0=0$.  

Choosing $\ell_d(e_{ij}, r_{IJ})=1 - \mathbb{1}_{\mathrm{sign}(e_{ij})=\mathrm{sign}(r_{IJ})},\forall i\neq j$,\footnote{The absence edge does not have the same sign as a directed edge, i.e., $\mathrm{sign}(0)\neq \mathrm{sign}(p), \forall p\neq 0$.} we describe the reconstruction by recovering the directed edge direction. We say a graph summarization has an exact \textit{Directions Influence Preserving Structure} (D-IPS) if there exists a summarization such that the reconstruction error based on $\ell_d$ is $0$.
%Note that the $\ell$ is not characteristic here. 

In a weighted graph, we may not only preserve the directional information of edges but also the weight information. Hence, we may choose the square loss between edges and compressed relations: $\ell_w(e_{ij}, r_{IJ})=(e_{ij}-r_{IJ})^2,\forall i\neq j$. We say a graph summarization has an exact \textit{Weights Influence Preserving Structure} (W-IPS) if there exists a summarization such that the reconstruction error based on $\ell_w$ is $0$. D-IPS is a special case for W-IPS. For a uniformly weighted graph, an exact W-IPS is equivalent to an exact D-IPS.

When a graph does not have an exact IPS, which is commonly observed in practice, we would like to simultaneously learn a node-compression, $\phi_V$, and an edge-compression, $\phi_E$, such that the corresponding summarization minimizes the relevant reconstruction error $L_0$.

\iffalse
\begin{lemma}
The edge information can be fully recovered if and only if $L_{0}(\phi; G(V,E)) = 0$ for characteristic distance $\ell$.
\end{lemma}
\begin{proof}
The if direction follows  from the definition. On the other direction, $L_{0}(\phi; G(V,E))=0$, $\ell(e_{ij}, r_{IJ})=0,\forall i,j$ as distance measure is non-negative. Characteristic distance gives $e_{ij} = r_{IJ}, \forall i,j$, which implies the fully recoverable edge information. 
\end{proof}
Note that an exact IPS does not imply the fully-recoverable edge-set since the distance used is not characteristic. Only edge direction is fully recoverable in an exact IPS. 
To extract the most representative summarization in a general directed graph, we would like to find the summarization $\phi$ that minimize the IPE, which preserve the influence of original graph as much as possible.
\fi

\section{Learning the Influence-Preserving Summarization} \label{sec:ips}
In this section, we present the formulation of  influence-preserving summarization as a constrained supervised learning objective based on the reconstruction error. Our ``labels" can be seen as the compressed relations. We then present the algorithm to solve the constrained optimization problem.

\subsection{The Constrained Supervised Learning Objective}
We start from defining our learning objective based on the reconstruction loss with $\ell_w$ and derive the factorization model as our constrained optimization objective.  
\paragraph{The IPS-based Objective}
Our objective is to seek a graph summarization of size $k$, that minimizes the reconstruction error (which corresponds to $\ell_w$ for the rest of the paper). Denote node-compression $\phi_V$ by assignment matrix $U\in\{0, 1\}^{n\times k}$ where the $I^{th}$ column vector $u_{:I}\in \{0,1\}^{n\times 1}$ represents the elements in compressed node $c_I$, i.e., $u_{iI}=\mathbb{1}_{x_i\in C_I}$. Denote the edge-compression by relationship matrix $R\in \R^{k\times k}$ where $r_{IJ}$ represents the compressed relation from $c_I$ to $c_J$. Since the summarized graph is assumed to be simple, $R$ is an asymmetric adjacency matrix.
Given weighted asymmetric adjacency matrix $A$ and a graph summarization represented by $U$ and $R$, the objective based on loss measure $\ell_w( A_{ij}, r_{IJ}) = (A_{ij}-r_{IJ})^2\mathbb{1}_{x_i\in C_I}\mathbb{1}_{x_j\in C_J}$ can be written as: 
\begin{equation}\label{eq: fix-obj}
L_{1}(A, U, R) =\sum_{I,J}{ \sum_{i,j} {(A_{ij} - r_{IJ})^{2}u_{iI} u_{jJ} }}.
\end{equation}
However, without information on the number of compressed nodes allowed, the objective in Eq.~\eqref{eq: fix-obj} will take $k=|V|$ and the zero reconstruction error can always be achieved. To avoid this, we would like to impose a constraint on the size of the summarized graph to make $k\ll|V|$. With such a constraint, this objective may still identify a compressed node containing less relevant elements %since $A$ is non-negative
. To address this problem, we propose a normalized version of the objective in Eq.~\eqref{eq: fix-obj}:  

\begin{equation}\label{eq: normalize-obj}
L_{2}(A, U, R)  = \sum_{I,J}{\frac{1}{|C_I||C_J|} \sum_{i,j} {(A_{ij}-r_{IJ})^{2}u_{iI} u_{jJ} }}
\end{equation}
which corresponds to a normalized loss measure: 
$\ell_w(A_{ij},r_{IJ}) = \frac{(A_{ij}-r_{IJ})^2\mathbb{1}_{x_i\in C_I}\mathbb{1}_{x_j\in C_J}}{|C_I||C_J|}$.
We further assume the compressed node does not have overlaps, which corresponds to the orthogonality constraints, i.e., $u_{:I}^{\T}u_{:J}=0,\forall I\neq J$.

\begin{lemma}\label{lem: equivalent-factorization}
The objective in Eq.~\eqref{eq: normalize-obj} 
%corresponds to $$L_{2}(A, U, R)  =\min_{U^{\T}U=I_{k},R} \sum_{I,J}\sum_{i,j} {(A_{ij}-r_{IJ})^{2}u_{iI} u_{jJ} }$$
has the factorization form
\begin{equation}\label{eq: factorization-obj}
L_{2}(A,U,R) = \|A - URU^{\T} \|^2_\mathrm{F},\quad s.t.\quad U^{\T}U=I_{k}; R_{IJ}R_{JI}=0,\forall I,J.
\end{equation}
\end{lemma}
The proof proceeds by basic linear algebra, which can be found in Appendix \ref{sec:additional_proofs}.
Note that $R$ is an asymmetric adjacency matrix representing the compressed relations in the summarized graph. Since the summarized graph is assumed to be simple, $r_{IJ}$ and $r_{JI}$ have at most one non-zero $\forall I\neq J$ and $R_{II}=0,\forall I$, which imply the constraint $R_{IJ}R_{JI}=0,\forall I,J$.

\paragraph{Continuous Relaxation}
The normalized objective in Eq.~\eqref{eq: normalize-obj} is an NP-hard discrete problem, which is similar to the discrete cluster assignment problem. Using continuous relaxation proposed in \cite{shi2000normalized} and \cite{meilua2007clustering} is a way to approximately solve such a problem. Here, we propose a continuous relaxation for the factorization model in Eq.~\eqref{eq: normalize-obj}:
\begin{equation}\label{eq: continuous-obj}
L_{3}(A, U, R) = \|A - URU^{\T} \|^2_\mathrm{F} \quad s.t. \quad U^{\T}U=I_{k}; {R_{IJ}R_{JI}=0,\forall I,J},
\end{equation} 
where $U\in \R^{n\times k}$ and $R\in \R^{k\times k}$.

Due to the constraint on $R$, it is not easy to solve such constrained objective as we do not assume structures on summarized graph. This issue can be alleviated by modeling  structure via the skew-symmetric adjacency matrix $T$. The corresponding factorization becomes: 
\begin{equation}\label{eq: skew-factorization}
L_{4}(T; U, S) =  \|T - USU^{\T} \|^2_\mathrm{F}\quad s.t. \quad U^{\T}U=I_{k},
\end{equation}
where $U \in \R^{n \times k}$ and $S \in \R^{k \times k}$ is skew-symmetric. 
Eq.~\eqref{eq: skew-factorization} exploits skew-symmetric structure and is easier to solve.
We show in Theorem \ref{Thm: equivalence-decomposition} below that the objectives in Eq.~\eqref{eq: continuous-obj} and  Eq.~\eqref{eq: skew-factorization} admit the same solution, up to permutation in the exact D-IPS case. Despite the fact that the asymmetric matrix $A$ is useful for deriving identifiability result, using the skew-symmetric matrix $T$ is easier to solve and more robust in noisy cases as the model explicitly penalizes the reversely directed noise edges. %show in Appendix \ref{sec:add-figures}%
In the rest of the paper, we will use the skew-symmetric matrix $T$ to represent the graphs.

\paragraph{Positive Values Identifies the Compressed Node}

With the factorization model in Eq.~\eqref{eq: skew-factorization}, we further show in Theorem \ref{Thm: exact IPS-Identification} of Section \ref{subsec: identi-analysis}, that under the exact D-IPS, the factor $U$ is non-negative and positive entries correctly identify the compressed nodes. Hence, we propose a non-negative constrained factorization model for better identification in the presence of noise:
\begin{equation}\label{eq: nn-skew-factorization}
L_{5}(T; U, S) = \|T - USU^{\T} \|^2_\mathrm{F} \quad s.t. \quad U\geq 0,U^{\T}U=I_{k}.
\end{equation}

\subsection{Learning Algorithms}

With the non-negative and orthogonal constraints on $U$, the model in Eq.~\eqref{eq: nn-skew-factorization} 
%is closely related to Orthogonal Non-negative Matrix Tri-Factorization \citep{ding2006orthogonal}, where 
can be written as a regularized version of the orthogonality constraint non-negative matrix tri-factorization:
\begin{equation}\label{eq:stnmf-obj}
L_{6}(T; U, S, \Lambda)=\|T-USU^{\T}\|^{2} + \mathrm{tr}(\Lambda(U^{\T}U-I)) \quad s.t. \quad U\geq0,
\end{equation}
where the  regularization parameter $\Lambda$ is a symmetric matrix. 
It is also related to Semi Non-negative Matrix Factorization %\citep{ding2010convex} 
since $T$ itself is not a non-negative matrix.

This optimization objective can be solved by gradient methods with projection to the
Stiefel manifold, as discussed in \cite{hirayama2016characterizing} and \cite{edelman1998geometry}.
However, the projection based algorithm is very sensitive to initialization.
Instead, we propose a multiplicative update scheme: $U\leftarrow U\odot\frac{[\nabla_{U}L]_{+}}{[\nabla_{U}L]_{-}}$ modified
from \cite{ding2010convex, lee2001algorithms} and \cite{ding2006orthogonal}. We use $X_{+}$ and $X_{-}$ to denote the positive and negative
parts of matrix $X$ respectively. The modification does not only allow the  imposition of the specific skew-symmetric structure
of $S$ and orthogonal constraint but also gives more stable results. This leads to our proposed Structured Non-negative Matrix Factorization (StNMF) in Algorithm~\ref{alg:MUR-fix-lambda}. 
\begin{algorithm}
\caption{StNMF (constant regularization parameter)\label{alg:MUR-fix-lambda}}
\textbf{Input:} Skew-symmetric $T$, size constraint $k$, regularizer $\Lambda$
\begin{algorithmic}[1] %[1] enables line numbers
\STATE Initialize: skew-symmetric $S$ and non-negative matrix $U$
\WHILE{not converged}
\STATE Update $U\leftarrow U\odot\frac{[TUS^{\T}]_{+}+U[S^{\T}U^{\T}US]_{-}}{[TUS^{\T}]_{-}+U([S^{\T}U^{\T}US]_{+}+\Lambda)}$
\STATE Update $S\leftarrow S\odot\frac{U^{\T}TU}{U^{\T}USU^{\T}U}$
\ENDWHILE
\STATE \textbf{return} $U$,$S$
\end{algorithmic}
\end{algorithm}

The non-negative matrix $U$ can be effectively initialized via non-negative
SVD \citep{boutsidis2008svd}. $S$ can be initialized by any $k\times k$ skew-symmetric matrix. For instance, when $k=2$, we set initial $S=\begin{pmatrix}0 & 1\\
-1 & 0
\end{pmatrix}$. The algorithm exploits zero locking
properties in the multiplicative update scheme so that the desired structure
of $S$ is preserved throughout the updates with such initialization. In the fixed regularization scheme, a different choice of $\Lambda$ results in a different local optimal solution and it depends on users' preference. For instance, if the user would like to have a strictly non-overlapping compressed node set, one may set the magnitude of off-diagonal terms to be large to emphasize orthogonality; if the user is interested in the weight assignment between compressed nodes, the diagonal terms may be set relatively larger to ensure the unit length vector. However, the discussion of such a topic is out of the scope in this paper. 

It is important to note that the directed compressed relations
can be read off from $S$, which represents the skew-symmetric (weighted) adjacency matrix of the summarized graph $H=(C,R)$.
Hence, our model is able to simultaneously identify the node-compression and the edge-compression, thus the summarized graph $H$.
%To obtain multiple slices of IPS compression, we can also perform deflation on original matrix as discussed in \cite{hyvarinen2016orthogonal, hirayama2016characterizing}. 

We can also optimize $\Lambda$, using the Karush-Kuhn-Tacker (KKT) complementary condition \citep{kuhn1951nonlinear} and set: $\Lambda = U^{\T}Q - P = U^{\T}Q_{+}+P_{-}-U^{\T}Q_{-}-P_{+}$,
where $ P=S^{\T}U^{\T}US$ and $Q=T^{\T}US$. The derivation can be found in Section \ref{subsec: algo-analysis}.
In addition, the Algorithm \ref{alg:MUR-fix-lambda}, does not guarantee a tightly bounded norm of the column vectors in $U$. When adaptive regularizer is used, the optimization trajectory is not monotonic non-increasing. Hence we impose a column-wise normalization step in Algorithm \ref{alg:adaptive-MUR}  to alleviate this problem.  

\begin{algorithm}
\caption{StNMF (the adaptive version)\label{alg:adaptive-MUR}}
\textbf{Input:} Skew-symmetric $T$, size constraint $k$
\begin{algorithmic}[1] %[1] enables line numbers
\STATE Initialize: skew-symmetric $S$ and non-negative matrix $U$
\WHILE{not converged}
\STATE Update $U\leftarrow U\odot\frac{[TUS^{\T}]_{+}+U[S^{\T}U^{\T}US]_{-}}{UU^{\T}[TUS^{\T}]_{+}+U[S^{\T}U^{\T}US]_{-}}$
\STATE Normalize column of $U$: $U\leftarrow U D_{U^{\T}U}^{-\frac{1}{2}}$
\STATE Update $S\leftarrow {U^{\T}TU}$
\ENDWHILE
\STATE \textbf{return} $U$,$S$
\end{algorithmic}
\end{algorithm}
%\vspace{-.8cm}

\section{Theoretical Analysis}\label{sec:theory}
In this section, we present theoretical results for identification of non-negative models and analysis of Structured Non-negative Matrix Factorization (StNMF). %adapted from Non-negative Matrix Factorization (NMF) 

\subsection{Identifiability Analysis}\label{subsec: identi-analysis}
\begin{theorem}\label{Thm: exact IPS-Identification}(exact D-IPS Identification)
Let $A$ be an asymmetric adjacency matrix of a directed graph with the exact D-IPS with $k$ compressed nodes. Assume that each submatrix between compressed nodes has distinct leading singular values with geometric multiplicity one. The optimization problem in Eq.~\eqref{eq: continuous-obj} has a unique solution $U \in \R^{n\times k}$ such that $U\geq 0$ and the positive part of each column vectors in $U$ identifies compressed nodes.
\end{theorem}
The proof technique extends on the $k=2$ case in Theorem \ref{Thm:Bipartitie-Identification}, which applies Perron-Frobenius Theorem on rearranged block matrix. Details can be found in Appendix \ref{sec:additional_proofs}. 
\iffalse
For an exact IPS, we can always rearrange the vertices such that $A=\begin{pmatrix}0 & 0\\
\tilde{A} & 0
\end{pmatrix}$. We show for such  $\tilde{A}^{\T}$, $\tilde{A}^{\T}\tilde{A}$ and $\tilde{A}\tilde{A}^{\T}$ are both primitive. Using Perron-Frobenius Theorem, we have a unique real positive leading eigenvector. Padding with $0$s, we have the unique non-negative solution. By reading out the position of non-zero term in $U$, we can determine the direction of the compressed relations. Detailed proof can be found in Appendix \ref{sec:additional_proofs}.
\fi
For graphs with more than one connected components to be determined, the most strongly connected component will be identified first and the consecutive components can be identified via deflation methods discussed in \cite{hyvarinen2016orthogonal,hirayama2016characterizing}, which is out of the scope of this paper.

\begin{lemma}\label{lem: block-form}
If a directed graph, with asymmetric adjacency matrix $A$, has the exact IPS, $A$ can be divided into block submatrix according to compressed nodes, such that: 1) If a block $\tilde{A}_{IJ}$ is non-zero, its block-wise transpose $\tilde{A}_{JI}$ zero matrix; 2) The diagonal blocks are zero-matrices.
\end{lemma}
\begin{proof}
By definition of the exact D-IPS, the direction of edges between compressed nodes are the same and there are no links within the compressed nodes. The result follows since summarized graph is simple.
\end{proof}
%\vspace{-0.3cm}
\begin{theorem}\label{Thm: equivalence-decomposition}(equivalence decomposition of $A$ and $T$)
Let $A$ and $T$ be the asymmetric and skew-symmetric adjacency matrix of a directed graph with the exact D-IPS, respectively. The optimal solution $U$ for $L_3$ and $L_4$ are the same up to permutation.
\end{theorem}
Detailed proofs can be found in Appendix \ref{sec:additional_proofs}.

\begin{corollary}
Let $T$ be the skew-symmetric adjacency matrix of a directed graph with the exact D-IPS. Assume that each submatrix between compressed nodes has geometric multiplicity two. The optimization problem with loss $L_4$ has a unique solution such that $U\geq 0$ and the positive part of each column vectors in $U$ identifies communities.
\end{corollary}
The corollary follows from combining Theorem \ref{Thm: exact IPS-Identification} and Theorem \ref{Thm: equivalence-decomposition}.

\subsection{Convergence Analysis}\label{subsec: algo-analysis}
In this section, we analyze the convergence of fixed regularization scheme in Algorithm \ref{alg:MUR-fix-lambda} and the adaptive regularization scheme with column-wise normalization in Algorithm \ref{alg:adaptive-MUR}.
Optimizing the objective in Eq.~\eqref{eq:stnmf-obj}, for fixed $\Lambda$ can be written as
\begin{equation}\label{eq:stnmf-obj-2}
\min_{U\geq0}L_{6}(T; U, S, \Lambda)=\min_{U\geq0}\mathrm{tr}(-2U^{\T}T^{\T}US + US^{\T}U^{\T}USU^{\T} + U\Lambda U^{\T} )
\end{equation}

\begin{lemma}
\label{lem:aux-fun}
Let $Q=T^{\T}US$ and $P=S^{\T}U^{\T}US$.
$$Z(U,U')=\mathrm{tr}(-2Q_{+}U^{\T} - UP_{-}U^{\T})+ \sum_{ij} \frac{[U'(P_{+}+\Lambda)]_{ij}U_{ij}^2}{U_{ij}'} + 2\left[Q_{-}\right]_{ij} \frac{U_{ij}^{2}+U_{ij}'^{2}}{2U_{ij}'}$$
is an auxiliary function of Eq.~\eqref{eq:stnmf-obj-2}.
\end{lemma}
The proof is based on pairing symmetric terms and details can be found in Appendix \ref{sec:additional_proofs}.

\begin{lemma}\label{lem:adpative-aux-fun}
Choosing $\Lambda = U^{\T}Q_{+} + P_{-} - U^{\T}Q_{-} - P_{+} $, the objective in Eq.~\eqref{eq:stnmf-obj-2} becomes:
\begin{equation}\label{eq: adaptive-objective}
\min_{U\geq0}L_7(T; U, S, \Lambda)=\min_{U\geq 0} \mathrm{tr}(-2U^{\T}Q_{+}-U^{\T}UP_{-}+U^{\T}U(U^{\T}Q_{+} + P_{-} - U^{\T}Q_{-})+2U^{\T}Q_{-})
\end{equation}
then 
\begin{equation}\label{eq:adaptive_aux}
Z(U,U')=\mathrm{tr}(-2Q_{+}U^{\T} - U^{\T}UP_{-})+ \sum_{ij} \frac{[U'(U'^{\T}Q_{+} + P_{-})]_{ij}U_{ij}^2 + [Q_{-}U'U'^{\T}]U'^2_{ij}}{U_{ij}'}    
\end{equation}
is an auxiliary function of Eq.~\eqref{eq: adaptive-objective}.
\end{lemma}

\begin{theorem}\label{thm:non-increasing}
The update rule described in Algorithm \ref{alg:MUR-fix-lambda} is non-increasing and converges to the stationary point of objective in Eq.~\eqref{eq:stnmf-obj} .
\end{theorem}
Proof of Theorem \ref{thm:non-increasing} can be found in Appendix \ref{sec:additional_proofs}. The proof technique is based on some carefully chosen symmetric matrices, which is applicable for fixed $\Lambda$ as it is a symmetric matrix. With the adaptively chosen $\Lambda$ in each step, as in Eq.~\eqref{eq:Lambda-val}, the column of $U$ does not have a fixed norm and $P_{+}+\Lambda$ is no longer symmetric, making the proof technique for Theorem \ref{thm:non-increasing} not applicable. However, with the proposed column-wise normalization scheme, the Algorithm is shown to be monotonic non-increasing and convergent to the stationary point. 

\begin{theorem}\label{thm:adaptive-non-increasing}
The update rule described in Algorithm \ref{alg:adaptive-MUR} is non-increasing and convergent.
\end{theorem}
The proof is by constructing a symmetric matrix based on the unit-normed column vectors and applying the proof techniques in Theorem \ref{thm:non-increasing}. Details can be found in Appendix \ref{sec:additional_proofs}.

\section{Experiments and Results} \label{sec:exp}
\begin{figure}
\vspace{-2ex}
\hspace*{-1.5cm} \includegraphics[scale=0.3]{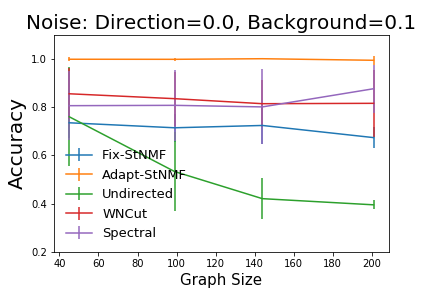}\includegraphics[scale=0.3]{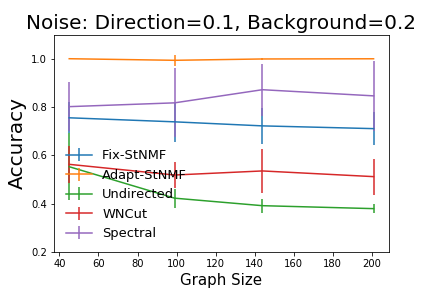}\includegraphics[scale=0.3]{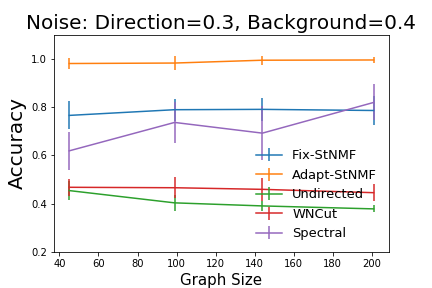}\includegraphics[scale=0.3]{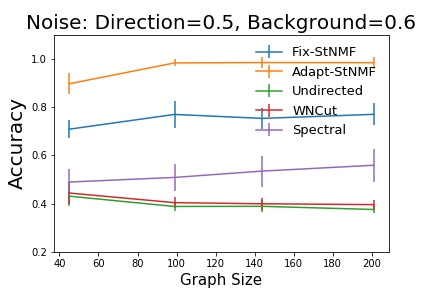}
\hspace*{-1.5cm}   \includegraphics[scale=0.3]{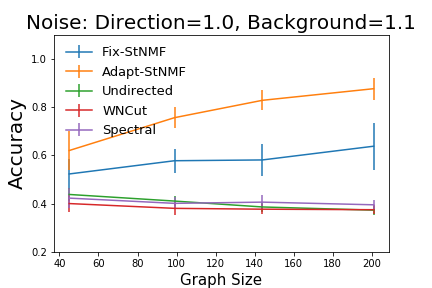}\includegraphics[scale=0.3]{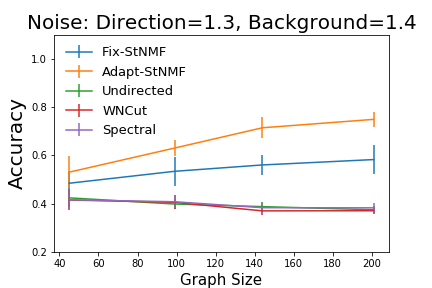}\includegraphics[scale=0.3]{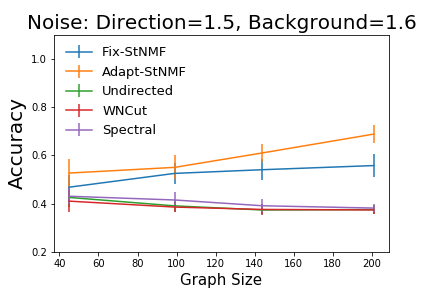}\includegraphics[scale=0.3]{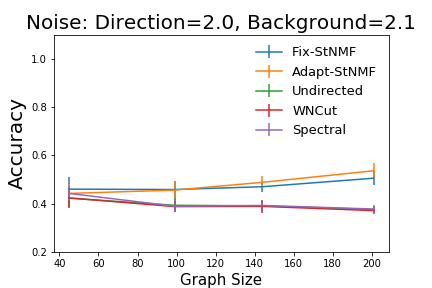}
    \caption{Compressed Node Assignment Accuracy with Different Noise Level}
    \label{fig:synethetic_result}
\vspace{-1ex}
\end{figure}

In this section, we apply the graph summarization model on synthetically generated directed graphs and compare with summarization methods on the undirected cases as well as conventional clustering such as spectral methods or normalized cut methods.
In the synthetic examples, we simulate graphs of different sizes at different noise levels with known compressed node. The background noise is the ratio: $\gamma_{b} =\frac{\sum_{i,j\notin \mathrm{D-IPS}}|e_{i'j'}|}{\sum_{i,j\in \mathrm{D-IPS}}|e_{ij}|}$ where the direction noise is the ratio: $\gamma_{d} =\frac{\sum_{i,j\in \mathrm{D-IPS}, e_{ij}<0}e_{ij}}{\sum_{i,j\in \mathrm{D-IPS}, e_{ij}\geq 0}e_{ij}}<0.5$. We compare the following algorithms: Fix-StNMF is the fixed regularization scheme in Algorithm \ref{alg:MUR-fix-lambda} and $\Lambda$ is chosen as a scalar time all-one matrix; Adaptive-StNMF is the adaptive scheme described in Algorithm \ref{alg:adaptive-MUR}; Undirected is the graph summarization scheme using the undirected skeleton, similar to \cite{hirayama2016characterizing}; WNCut is the weighted normalized cut scheme \citep{meilua2007clustering} for directed graph; Spectral is the clustering method using normalized Laplacian \citep{shi2000normalized}. 
From the result, we see that at low noise levels, the Adaptive-StNMF correctly finds the compressed node assignment, as the theory shows. When the noise level is higher, it still performs best among the competitors. The low accuracies for ``clustering methods" are expected as they do not maximize the desired objectives. Moreover, we see that the Fixed-StNMF is worse than the Adaptive version as we deliberately chose $\Lambda$ to be an all-one matrix, where the algorithm does not necessarily converge to the most useful local optimal, which shows that the learning accuracy is also sensitive to the choice of regularization parameters.  

\section{Conclusion and Future Work}\label{sec:conclusion}
We propose a new problem setting to summarize directed graphs. Our key contribution is to define a novel learning criterion that preserves the directed edge information from the original graph. Our criterion is related to the \textit{reconstruction error} from the summarized graph to the original graph. We proposed a non-negative algorithm to learn such graph summarization. We provide theoretical analysis on identifiability and convergence and experimental demonstration to show the usefulness of our method. 

% \pagebreak

\bibliographystyle{plainnat}
%\small 
\bibliography{main}

\pagebreak
%\onecolumn

%\acks{Acknowledgements should go at the end, before appendices and references.}

\appendix
\iffalse
\begin{center}
\Large
Supplementary material for Direction Matters: \\ On Influence-Preserving Graph Summarization and \\ Max-cut Principle for Directed Graphs
\end{center}
\fi

%\sectionfont{\small}
\section{Additional Theorems and Proofs}\label{sec:additional_proofs}
%%% Analysis of Objective 
\subsection*{Proof of Lemma \ref{lem: equivalent-factorization}}
\begin{proof}
Normalized by the size of compressed node, each assignment vector has unit length.  
Expanding each term in Eq.~\eqref{eq: fix-obj}, we have
$$LHS = \sum_{I,J}\sum_{i,j}A_{ij}^2u_{iI}u_{jJ} - 2A_{ij}r_{IJ}u_{iI}u_{jJ}+r_{IJ}^2u_{iI}u_{jJ}
$$
Summing over the index $I,J$, the first term is $\sum_{i,j}A_{ij}=\mathrm{tr}(A^{\T}A)$, which is a constant independent of $R$ and $U$. Summing over the index $I,J$, the second term is $\sum_{i,j}A_{ij}(URU^{\T})_{ij}=\mathrm{tr}(A^{\T}URU^{\T})$. For the third term, since $\sum_{i,j} u_{iI}u_{jJ}=1$ in the normalized setting, summing over index $i,j$, we have $\sum_{I,J}r_{IJ}^2=\mathrm{tr}(R^{\T}R)$.
Writing 
$$RHS = \mathrm{tr}( - 2A^{\T}URU^{\T} + R^{\T}R),$$ 
the result follows. As the summarized graph is simple, the constraint on $R$ in factorization model is imposed.
\end{proof}
%%% Analysis for identifiability
\begin{proposition}
%\vspace{-.5cm}
\label{(Perron-Frobenius-Theorem)}(Perron-Frobenius) Suppose $M\in\R^{n\times n}$ is
a non-negative square matrix that is irreducible, then:
\begin{enumerate}
\item $M$ has a positive real eigenvalue $\lambda_{\max}$, such that all
other eigenvalues of $M$ satisfy, $|\lambda|\leq\lambda_{\max}$
(if $M$ is primitive, $|\lambda|<\lambda_{\max}$)
\item $\lambda_{max}$ has algebraic and geometric multiplicity $1$ and
has positive eigenvector $x>0$ (called Perron vector)
\item any non-negative eigenvector is a multiple of $x$
\end{enumerate}
\end{proposition}
%Perron-Frobenius Theorem is a well-known result and various proof strategies can be found in e.g. \cite{maccluer2000many}. 

\begin{proof}
$M$ is irreducible non-negative square matrix, then $\exists k\in\mathbb{N}^{+}$
such that $P=(I+M)^{k}>0$. $(I+M)^{k}=I+M+\frac{1}{2!}M^{2}+...\frac{1}{k!}M^{k}$.
By irreducibility and non-negativity, for large enough $k$, the expansion
fills in all $n^{2}$ terms with positive numbers. Hence $P$ is primitive.
We also have $TP=PT$.

Let $Q$ be the positive orthant and $C$ be the intersection of the
surface of the unit sphere and positive orthant. $\forall z\in Q$,
define a function:
\[
L(z)=\max\{s:sz\leq Tz\}=\min_{1\leq i\leq n,z_{i}>0}\frac{(Tz)_{i}}{z_{i}}
\]

For $\forall r>0$, we have $L(rz)=L(z)$ by definition, so $L(z)$
depends only on the ray along $z$.

We write $\leq$ sign between vectors, $v\leq w$ to imply $v_{i}\leq w_{i},\forall i$.
Similar definition applies for $<$. For $v\leq w$ and $v\neq w$,
we have $Pv<Pw$, since $P(w-v)\geq0$ and $P(w-v)\neq0$.

If for scalar $s$, $sz\leq Tz$, then $Psz\leq PTz=TPz$, which implies
$s(Pz)\leq T(Pz)$. Thus, $L(Pz)\geq L(z)$.

If $L(z)z\neq Tz$, then $L(z)Pz<TPz$. This implies $L(z)<L(Pz)$,
unless $z$ is an eigenvector ($Tz=L(z)z$) Hence, positive $z$ is
eigenvector when $L(z)$ is maximised.

Consider the image of $C$ under $P$. It is compact as it is the
image of a compact set under a continuous map. All of the elements
of $P(C)$ have all their components strictly positive, as $P>0$.
Hence the $L$ is continuous on$P(C)$. Thus $L$ achieves a maximum
value on $P(C)$. Since $L(z)\leq L(Pz)$, this is, in fact, the maximum
value of $L$ on all of $Q$, which implies the existence of maximum
eigenvalue. Since $L(Pz)>L(z)$ unless $z$ is an eigenvector of $T$,
$L_{max}$ is achieved at an eigenvector, call it $x$ of T and $x>0$
with $L_{max}$ as the eigenvalue. Since $Tx>0$ and$Tx=L_{max}x$
we have $L_{max}>0$.

Let $y$ be any other eigenvectors of $T$ with eigenvalue $\lambda$,
we have $\lambda y_{i}=\sum_{j}T_{ij}y_{j}$. As $T\geq0$, we have
$|\lambda y_{i}|=\sum_{j}T_{ij}|y_{j}|$, thus we write $|\lambda||y|\leq T|y|$.
Consider $|\lambda|\leq L(|y|)\leq L_{max}$ by definition of $L$,
writing $\lambda_{max}=L_{max},$ we show that $|\lambda|\leq\lambda_{max}$.
Note that if $\lambda_{max}=0$, $T$ is nil-potent, contradicting
to irreducible. Thus we have $\lambda_{max}>0$.

Consider the rate of change in characteristic polynomial of matrix
$T$:

\[
\frac{d}{d\lambda}det(\lambda I-T)=\sum_{i}det(\lambda I-T(i))
\]
where $T(i)$ is matrix $T$ deleting $i^{th}$ row and column. Each
of the matrices $\lambda_{max}I-T(i)$ has strictly positive determinant,
which shows that the derivative of the characteristic polynomial of
$T$ is not zero at $\lambda_{max}$, and therefore the algebraic
multiplicity and hence the geometric multiplicity of $\lambda_{max}$
is one.

If there exists any other nontrivial non-negative eigenvector $y\geq0$,
such that $y$ is not a multiple of $x$, since $\lambda_{max}$ has
geometric multiplicity $1$, $y^{\T}x=0$. However, $x>0$ and $y^{\T}x=0$
implies $y=0$, a contradiction.
\end{proof}

%\vspace{-.5cm}
\subsection*{Proof of Theorem \ref{Thm: equivalence-decomposition}}
\begin{proof}
By Lemma \ref{lem: block-form}, we write $A$ in the block form where the blocks are grouped by compressed node assignment. Hence, use the fact that  $\mathrm{tr}(AA)=0=\mathrm{tr}(RR)$ from simple graph and $U^{\T}AU$ has the same zero/non-zero positions as $R$ for the exact D-IPS, we have $\mathrm{tr}((A-URU^{\T})(A-URU^{\T}))=\mathrm{tr}(AA - 2U^{\T}AUR + RR)=0$ and 
$$\|T-USU^{\T}\|^2 = \|A-URU^{\T} - (A-URU^{\T})^{\T}\|^2 = 2\|A-URU^{\T}\|+0$$
Hence, both objectives are solving the same problem.
\end{proof}
\iffalse
\begin{lemma}\label{lem: diag-decomp}
Let $M$ be a block diagonal matrix of the form $M=\begin{pmatrix}M_1 & 0 & 0\\ 0 & ... & 0 \\
0 & 0 & M_k
\end{pmatrix}$ where submatrix $M_1$, ..., $M_k$ all have distinct positive leading eigen-vectors. Then the best rank $k$ orthogonal approximation of $M$ is the leading eigenvector of $M_1$ to $M_k$ padded with $0$.
\end{lemma}

\begin{proof}
The proof is by contradiction. Assume the eigen- vectors for the $k$ blocks of submatrices are $v_1,...,v_k$, with eigenvalues $\lambda_1,...,\lambda_k$. Let $u_i$ be $v_i$ positioned on $i^{th}$ block padded with $0$. Hence, the eigenvector of $M$, with eigenvalue $\lambda$ takes the form $u=\sum a_i u_i$, where $\sum a_{i}^2 = 1$ since $\|u\|=1$. Hence, $u^{\T}Mu=\sum a_i^2\lambda_i$. For distinct $\lambda_i$, the optimal value is $\max_{i}\lambda_i$ due to the unit norm constraint. For the second vector, we simply consider the case that the submatrix of $M$ where the largest eigenvalue block submatrix was removed. 
\end{proof}
\fi
%\vspace{-.5cm}
\begin{theorem}
\label{Thm:Bipartitie-Identification}(Bipartite Identification)
Let $A$ be an asymmetric adjacency matrix of the exact D-IPS of two compressed nodes. SVD of $A$ has a unique leading
left and right singular vector $v,w\geq0$ and the positive part of
$v,w$ identifies two compressed nodes.
\end{theorem}
\begin{proof}
For the exact D-IPS with two compressed nodes, we can always rearrange the vertices such that $A=\begin{pmatrix}0 & 0\\
\tilde{A} & 0
\end{pmatrix}$. $\tilde{A}^{\T}\tilde{A}$ and $\tilde{A}\tilde{A}^{\T}$ represents the "in-out" and "out-in" two step transition. As the two compressed nodes are connected, the any vertex from the two step transition can reach any other vertex in the same compressed node. Hence, $\tilde{A}^{\T}\tilde{A}$ and $\tilde{A}\tilde{A}^{\T}$ are both primitive. Using Perron-Frobenius Theorem, we have a unique real positive leading eigenvector. Padded with $0$s, the leading eigenvectors of $\tilde{A}^{\T}\tilde{A}$ and $\tilde{A}\tilde{A}^{\T}$ are unique and non-negative where the non-zero terms corresponds to the compressed node assignment.
\end{proof}{}

\subsection*{Proof  Theorem \ref{Thm: exact IPS-Identification}}
\begin{proof}
The proof of Theorem \ref{Thm: exact IPS-Identification} is based on Proposition \ref{(Perron-Frobenius-Theorem)}, and Lemma \ref{lem: block-form}. Re-arrange the indices according to compressed node and denote the block submatrix between $C_I$ and $C_J$ as $\tilde{A}_{IJ}\in \R^{|C_I|\times|C_J|}$. Write $\Bar{A}_{IJ}\in \R^{n\times n}$ as the zero-padded matrix of $\tilde{A}_{IJ}$. The zero-padded vector for compressed node $C_I$, denoted by $u^{I}$ is the vector with non-zero $i^{th}$ entries for $x_i\in C_{I}$ and zeros otherwise.
Write each column of $U$, $u_{:I'}$ as a linear combination of zero-padded vector: $u_{:I'}=\sum_{I}\eta_{II'}u^{I}_{:I'}$, where $\sum_{I}\eta_{II'}^2=1$.  We write the non-zero part of $u^{I}\in \R^{n}$ as $u^{I}_I\in \R^{|C_I|}$, which is a unit vector.
The optimization objective in Eq.~\eqref{eq: continuous-obj} can be written as:
$$
-2\sum_{I',J'}r_{I'J'}u_{:I'}^{\T} (\sum_{I,J}\Bar{A}_{IJ})u_{:J'} + \sum_{I',J'}r_{I'J'}^2
$$
Differentiate w.r.t. $r_{I'J'}$ to find the optimized $r_{I'J'}=u_{:I'}^{\T} (\sum_{I,J}\Bar{A}_{IJ})u_{:J'}$, then the optimization objective becomes:
$\max_{u}\sum_{I',J'}(u_{:I'}^{\T} (\sum_{I,J}\Bar{A}_{IJ})u_{:J'})^2$
which can be simplified as $\sum_{I',J'}(\sum_{I,J}\eta_{II'}\eta_{JJ'}w^{IJ}_{I'J'})^2$ where $w^{IJ}_{I'J'}={u^{I}_{II'}}^{\T}\tilde{A}_{IJ}{u^{J}_{JJ'}}$. Since $u^{I}_{II'},u^{J}_{JJ'}$ are unit vectors, $max_{I'J'}w^{IJ}_{I'J'}\leq \lambda_{IJ}$ where $\lambda_{IJ}$ is the leading singular value of $\tilde{A}_{IJ}$. Due to the unit norm constraint, we have the objective 
$$\sum_{I',J'}(u_{:I'}^{\T} (\sum_{I,J}\Bar{A}_{IJ})u_{:J'})^2\leq \sum_{IJ}\lambda_{IJ}^2$$ 
where the equality holds when $u^{I}_{II}$, $u^{J}_{JJ}$ are the left and right singular vectors of $\tilde{A}_{IJ}$ and $\eta_{II'}=\mathbb{1}_{I=I'}$. By Theorem \ref{Thm:Bipartitie-Identification}, we know that $\tilde{A}_{IJ}$ are primitive for all $I,J\in [k]$. Applying Perron-Frobenius in Theorem \ref{(Perron-Frobenius-Theorem)}, $u^{I}_{II}>0$ and $u_{:I}\geq 0$ where the positive part identifies some compressed node $C_I$. As the compressed node blocks does not need to have an order, the solution is unique only up to permutation of blocks.
\end{proof}

%%%% Analysis for algorithm
\begin{proposition}\label{prop: prop6}
(Proposition 6 in \cite{ding2006orthogonal}) For any symmetric matrices A
$\in\R_{\geq 0}^{n\times n}, B\in\R_{\geq 0}^{k\times k},S,S'\in\R_{\geq 0}^{n\times k}$, the following inequality holds:
$\sum_{i,p}\frac{(AS'B)_{ip}S_{ip}^{2}}{S'_{ip}}\geq \mathrm{tr}(S^{\T}ASB)$
\end{proposition}

\begin{proof}
Write $S_{ip}=S'_{ip}a_{ip}$. Then $\sum_{i,p}\frac{(AS'B)_{ip}S_{ip}^{2}}{S'_{ip}}- \mathrm{tr}(S^{\T}ASB) = $
$$
\sum_{i,k,l,p}A_{ik}S'_{kl}B_{lp}S'_{ip}(a_{ip}^2-a_{ip}a_{kl})\\
=\sum_{i,k,l,p}\frac{1}{2}A_{ik}S'_{kl}B_{lp}S'_{ip}(a_{ip}^2+a_{kl}^2-2a_{ip}a_{kl})\geq 0
$$
as $A$ and $B$ are symmetric and non-negative.
\end{proof}

\begin{proposition}\label{prop: prop6-var}
For any matrices $ B\in\R_{\geq 0}^{k\times k},S,S'\in\R_{\geq 0}^{n\times k}$, and B is symmetric, the following inequality holds
$\sum_{i,p}\frac{(BS'^{\T})_{ip}S_{ip}^{2}}{S'_{ip}}\geq \mathrm{tr}(SBS^{\T})$
\end{proposition}
\begin{proof}
Similar to the proof above, we write Write $S_{ip}=S'_{ip}a_{ip}$. Then $$\sum_{i,p}\frac{(BS'^{\T})_{ip}S_{ip}^{2}}{S'_{ip}} - \mathrm{tr}(SBS^{\T}) = 
\sum_{i,k,l,p}B_{ik}S'_{pk}S'_{ip}(a_{ip}^2-a_{ip}a_{kp}) = 
\sum_{i,k,l,p}B_{ik}S'_{pk}S'_{ip}(a_{ip}^2+a_{kp}^2-2a_{ip}a_{kp})\geq 0
$$
as $B$ is symmetric and non-negative.
\end{proof}

\subsection*{Proof of Lemma \ref{lem:aux-fun}}
\begin{proof}
Write $Q = T^{\T}US$ and $P=S^{\T}U^{\T}US$. Since both $Q$ and $P$ are not non-negatie matrices in general, the optimization objective $L_6$ in Eq.~\eqref{eq:stnmf-obj-2} can be written as:
$$L_6(T; U, S, \Lambda)=  \mathrm{tr}(-2U^{\T}Q_{+}-U^{\T}UP_{-}+U^{\T}U(P_{+}+\Lambda)+2U^{\T}Q_{-})$$
for $U\geq 0$.
From Proposition \ref{prop: prop6}, we have $\mathrm{tr}(U(P_{+}+\Lambda)U^{\T})\leq\sum_{ij}\frac{[U'(P_{+}+\Lambda)]_{ij} U_{ij}^{2}}{U_{ij}'}$ since $P_{+}$ and $\Lambda$ are both symmetric matrices. Using $a\leq\frac{a^{2}+b^{2}}{2b}$, we have $\mathrm{tr}(Q_{-}U^{\T})\leq\sum_{ij}[Q_{-}]_{ij} \frac{U_{ij}^{2}+U_{ij}'^{2}}{2U_{ij}'}$. $Z(U,U')$ reaches lower bound $L_3$ when $U=U'$. Hence, $Z(U,U')$ is an auxiliary function.
\end{proof}

\subsection*{Proof of Theorem \ref{thm:non-increasing}}
\begin{proof}
Using the auxiliary function in Lemma \ref{lem:aux-fun}, we take the derivative of $Z(U,U')$ w.r.t. $U_{ij}$:
\begin{eqnarray*}
\hspace{-0.1in}
\frac{\partial Z(U,U')}{U_{ij}}=2\left([-Q_{+}-U'P_{-}]_{ij}  + \frac{[U'(P_{+}+\Lambda)+Q_{-}]_{ij}U_{ij}}{U'_{ij}}\right)=0
\end{eqnarray*}
Solving the stationary point, we have the update rule for $U$ as stated in Algorithm \ref{alg:MUR-fix-lambda}: $$U_{ij}=U'_{ij}\frac{[Q_{+}+U'P_{-}]_{ij}}{[U'[P_{+}+\Lambda] + Q_{-}]_{ij}}.$$
Since the update of S is independent of $\Lambda$, the update can be readily adapted from (Theorem 8 \cite{ding2006orthogonal}).
As the objective is bounded below and the iterative procedure is monotonic non-increasing, the algorithm finds the local minimum of the objective function.
\end{proof}
\begin{lemma}\label{lem:norm-trace}
Let $U\in \R^{n\times k}$ be orthogonal matrix such that $U^{\T}U=I_{k}$ and $U'\in \R^{n\times k}$ be a matrix of unit column vectors. Let $G \in \R^{k}$ be a non-negative matrix. Then $\mathrm{tr}(U^{\T}UG)\leq \mathrm{tr}(U'^{\T}U'G)$  
\end{lemma}
\begin{proof}
Write $U'^{\T}U'=I_{k} + E$ for some non-negative matrix $E$.  Since $E$ and $G$ are non-negative, then $\mathrm{tr}(U'^{\T}U'G)=\mathrm{tr}(I_{k}G+EG)\geq \mathrm{tr}(I_{k}G)$
\end{proof}

\begin{lemma}\label{lem:sym-form}
$U^{\T}Q=U^{\T}T^{\T}US$ is symmetric under the update rule of Algorithm \ref{alg:adaptive-MUR}.
\end{lemma}
\begin{proof}
Under the update rule in Algorithm \ref{alg:adaptive-MUR}, as $U$ is column-wise normalized, $U^{\T}T^{\T}U=S^{\T}$. Hence, $U^{\T}Q=S^{\T}S$ is symmetric.
\end{proof}
It is worth note that, the original scheme proposed in \cite{ding2006orthogonal}, without normalization does not have such property. Assume the norm for each row of $U$ is $D$, where normalized $\tilde{U}D = U$. Then the update $\tilde{S}=\tilde{U}^{\T}T\tilde{U}$, where $S=U^{\T}TU=D\tilde{S}D$. Hence, $S^{\T}S=D\tilde{S}^{\T}D\tilde{S}$ is not necessarily symmetric, which violate the auxillary function formulation.

\subsection*{Proof of Lemma \ref{lem:adpative-aux-fun}}
\begin{proof}
The proof is using Lemma \ref{lem:norm-trace}. Due to normalization step, the factor $U$ have unit norm column vectors. Hence, $\mathrm{tr}(U'^{\T}U'U^{\T}Q_{-})\geq \mathrm{tr}(U^{\T}Q_{-})$ and
$$\mathrm{tr}(-2U^{\T}Q_{+}-U^{\T}UP_{-}+U^{\T}U(U^{\T}Q_{+} + P_{-} - U^{\T}Q_{-}) + 2U'^{\T}U'U^{\T}Q_{-})$$
is an upper bound for Eq.~\eqref{eq: adaptive-objective}, where equality hold when $U$ is orthogonal matrix. 
As $U^{\T}Q$ is symmetric by Lemma \ref{lem:sym-form}, 
we can apply Proposition \ref{prop: prop6} and have 
$$\mathrm{tr}(U^{\T}U(U^{\T}Q_{+} + P_{-} - U^{\T}Q_{-}))\leq \sum_{ij}\frac{[U'(U'^{\T}Q_{+} + P_{-} - U'^{\T}Q_{-})]U_{ij}^2}{U'_{ij}}$$. 
We also have 
$$\mathrm{tr}(U^{\T}Q_{-}U'^{\T}U')\leq \sum_{ij}\frac{U_{ij}^2 + {U'_{ij}}^2}{2U'_{ij}} (U'U'^{\T}Q_{-})_{ij}$$. Combining both term, the result follows.
\end{proof}

We assume $\Lambda + P_{+}\geq 0$. The KKT condition on the orthogonal constraint case can be applied to choose the optimum regularization term $\Lambda$. The KKT condition reads:
\begin{equation}\label{eq:kkt}
2[-Q_{+}-UP_{-}+UP_{+}+Q_{-}+U\Lambda]_{ij}U_{ij}=0.
\end{equation}
For diagonal terms, we sum over $j$ in Eq.~\eqref{eq:kkt} to have 
$[-U^{\T}Q_{+}-U^{\T}UP_{-}+U^{\T}UP_{+}+U^{\T}Q_{-}+U^{\T}U\Lambda]_{ii}=0$, which implies $\Lambda_{kk}=[U^{\T}Q_{+}+P_{-}-P_{+}-U^{\T}Q_{-}]_{kk}$. For off diagonal terms $j\neq p$, $\sum_{k}[\Lambda+P]_{ik}U_{jk}=Q_{ij}$, multiply $U_{ip}$ and sum over $p$ on both sides, we get $\sum_{k}[\Lambda+P]_{pk}=[\Lambda+P]_{jp}=[U^{\T}Q]_{jp}$. Hence we have:
\begin{equation}\label{eq:Lambda-val}
\Lambda=U^{\T}Q-P=U^{\T}Q_{+}+P_{-}-U^{\T}Q_{-}-P_{+}
\end{equation}
with $\Lambda + P_{+}\geq 0$.
\subsection*{Proof of Theorem \ref{thm:adaptive-non-increasing}}

\begin{proof}
Applying KKT condition and choosing adaptive $\Lambda=U^{\T}Q - P$, the objective has the form in Eq.~\eqref{eq: adaptive-objective}, which is bounded by Eq.~\eqref{eq:adaptive_aux} in Lemma \ref{lem:adpative-aux-fun}. Differentiate Eq.~\eqref{eq:adaptive_aux} w.r.t. $U_{ij}$:
\begin{eqnarray*}
2\left([-Q_{+}-U'P_{-}]_{ij}  + \frac{[U'(P_{-}+U'^{\T}Q_{-})]_{ij}U_{ij}}{U'_{ij}}\right)=0
\end{eqnarray*}
Solving the stationary point, we have the update rule for $U$ as stated in Algorithm \ref{alg:MUR-fix-lambda}: $$U_{ij}=U'_{ij}\frac{[Q_{+}+U'P_{-}]_{ij}}{[U'[P_{+}+\Lambda] + Q_{-}]_{ij}}.$$
Since the $U$ factor here does not have unit norm for each column, we explicitly normalized $U$ and update $S=U^{\T}TU$ after normalization. With the normalization step, the optimization scheme in Algorithm \ref{alg:adaptive-MUR} is non-increasing even for the adaptive regularization scheme. Since the objective is bounded below, it converges to the stationary point.
\end{proof}

\end{document}